\documentclass[11pt]{article}

\usepackage[truedimen,margin=25truemm]{geometry}

\usepackage{subfigure} 
\usepackage{amsmath}
\usepackage{amsthm}
\usepackage{amssymb}
\usepackage{algorithm}
\usepackage{algorithmic}

\usepackage{bm}
\usepackage{bbm}
\usepackage{dsfont}
\usepackage{color}
\usepackage{ulem}
\usepackage{booktabs}
\usepackage[colorlinks=false]{hyperref}
\usepackage{graphicx}
\usepackage{epsfig}
\usepackage{subfigure}
\usepackage{mathtools}
\usepackage[raggedright]{titlesec} 
\usepackage[numbers]{natbib}
\usepackage{siunitx,etoolbox}

\usepackage[toc,page]{appendix}
\usepackage{chngcntr}
\usepackage{apptools}
\AtAppendix{\counterwithin{lemma}{section}}
\AtAppendix{\counterwithin{figure}{section}}
\AtAppendix{\counterwithin{table}{section}}

\newcommand*{\ldblbrace}{\{\mskip-5mu\{}
\newcommand*{\rdblbrace}{\}\mskip-5mu\}}

\newtheorem{lemma}{Lemma}
\newtheorem{proposition}{Proposition}

\usepackage{color}
\definecolor{harvardcrimson}{rgb}{0.79, 0.0, 0.09}

\usepackage{booktabs}
\usepackage{multirow}
\DeclareMathAlphabet{\mathbbmsl}{U}{bbm}{m}{sl}

\usepackage{pifont}

\usepackage{bibentry}

\title{Wasserstein Graph Distance Based on $L_1$-Approximated Tree Edit Distance between Weisfeiler-Lehman Subtrees}

\usepackage{authblk}
\author[1]{Zhongxi Fang \thanks{fzx@akane.waseda.jp}}
\author[1]{Jianming Huang \thanks{koukenmei@toki.waseda.jp}}
\author[1]{Xun Su \thanks{suxun\_opt@asagi.waseda.jp}}
\author[1,2]{Hiroyuki Kasai \thanks{hiroyuki.kasai@waseda.jp}}
\affil[1]{Department of Computer Science and Communications Engineering, WASEDA University}
\affil[2]{Department of Communications and Computer Engineering, WASEDA University}

\begin{document}
\maketitle

\begin{abstract}
The Weisfeiler-Lehman (WL) test is a widely used algorithm in graph machine learning, including graph kernels, graph metrics, and graph neural networks. However, it focuses only on the consistency of the graph, which means that it is unable to detect slight structural differences. Consequently, this limits its ability to capture structural information, which also limits the performance of existing models that rely on the WL test. This limitation is particularly severe for traditional metrics defined by the WL test, which cannot precisely capture slight structural differences. In this paper, we propose a novel graph metric called the Wasserstein WL Subtree (WWLS) distance to address this problem. Our approach leverages the WL subtree as structural information for node neighborhoods and defines node metrics using the $L_1$-approximated tree edit distance ($L_1$-TED) between WL subtrees of nodes. Subsequently, we combine the Wasserstein distance and the $L_1$-TED to define the WWLS distance, which can capture slight structural differences that may be difficult to detect using conventional metrics. We demonstrate that the proposed WWLS distance outperforms baselines in both metric validation and graph classification experiments.
\end{abstract}

\section{Introduction}
In recent years, the remarkable performance improvements of graph neural networks (GNNs) have triggered a surge of research on their applications in various domains, such as recommendation systems~\cite{gnnfrs} and drug and material discovery~\cite{10.1093/bib/bbab159, Takamoto2022}. At the same time, a critical need has arisen for accurate tools that can measure graph similarity and distance to enable effective graph sorting and analysis. However, comparing graph structures is a difficult problem that has been studied for decades~\cite{bunke1998graph, gao2010survey, pmlr-v97-titouan19a}. 

Graph edit distance (GED) is a classical approach to this problem. However, GED is NP-hard and still requires high time complexity, even with its well-known approximation algorithms. For instance, the popular A*-Beamsearch~\cite{fsacged} has sub-exponential time complexity. Learning-based methods such as SimGNN~\cite{singnn} combine GNNs and other neural networks to estimate the similarity between graphs. However, these methods require an accurate similarity score as a label, which limits their application scope. Additionally, it has been pointed out that GNNs cannot fully exploit the structural information of graphs~\cite{Errica2020A}. Random-walk-based graph embeddings, such as DeepWalk~\cite{perozzi2014deepwalk} and Node2Vec~\cite{grover2016node2vec}, provide another way of describing structural information. Although they can capture the regularity of node connections, they cannot handle previously unseen nodes due to their use of transductive learning. Furthermore, finding appropriate parameters for random walks can be costly.

In contrast, graph kernels~\cite{nikolentzos2021graph} are a class of methods that specialize in measuring the similarities of graph structures. Most of them are based on $\mathcal{R}$-convolutional theory~\cite{haussler1999convolution}, which computes graph similarity by decomposing a graph into subgraphs, measuring the similarities between subgraphs, and aggregating them. Some well-known graph kernels produce more stable and competitive classification results compared to GNNs. We aim to measure even slight differences in the entire graph structures by correctly measuring the differences between subgraphs. To this end, we delve into one of the most influential graph kernels, the Weisfeiler-Lehman (WL) subtree kernel~\cite{shervashidze2009fast}.

The WL subtree kernel, also known as the WL kernel, is a pioneering graph kernel that uses a neighborhood aggregation scheme. It was inspired by the WL test~\cite{weisfeiler1968reduction}, which provides an approximate solution to the graph isomorphism problem. Due to its stable and high performance in graph classification tasks and its similarity to the message-passing algorithm of GNNs, the WL kernel is often used as a baseline for GNNs~\cite{morris2019weisfeiler, bodnar2021weisfeiler, wijesinghe2022a}. Furthermore, previous studies have shown that the WL framework can provide high accuracy~\cite{nikolentzos2021graph}. However, we argue that the WL kernel's measure of graph similarity is coarse, and there are two main reasons for this. The first problem is that the ability to describe structural information is weak. This problem stems from the fact that WL test focuses only on the consistency of the graph, in particular the consistency of the subgraphs composed of a node and its neighborhood; the WL test projects different subgraphs to different integer values using the hash function and compares the results for subgraph matching, which results in the loss of specific information about the connections between nodes. The second problem is that the simplicity of the measure limits the expressive power of similarity. Graph kernels are typically computed from two parts: a node-level measure that measures the similarity of subgraphs and a graph-level measure that computes the similarity of entire graphs using subgraph similarities. The WL kernel measures the similarity between nodes by subgraph matching and then sums the similarities of all pairs of nodes to compute the graph similarity. To address the first problem, \citep{schulz2022generalized} proposed a relaxed WL kernel that defines the similarity between subgraphs more finely by treating similar subgraphs as identical. To address the second problem, \citep{Togninalli19} proposed the Wasserstein WL (WWL) distance that applies the Wasserstein distance~\cite{peyre2019computational} to the graph-level measure.

Motivated by the observations mentioned above, we aim to enhance the descriptive power of structural information without disrupting the mechanism of the WL test. Specifically, we introduce a WL subtree, a subgraph consisting of a node and its neighborhood structure, in accordance with the mechanism of the WL test. The WL subtree is a rooted unordered tree that corresponds to the node label obtained from the WL test. The concept of the WL subtree was originally proposed by~\citep{shervashidze2011weisfeiler}, and in earlier studies, WL subtrees were used only to interpret the WL kernel and analyze the expressive power of GNNs~\cite{sato2020survey}. In this paper, however, we treat them as structural information of node neighborhoods, which differentiates our proposed method from others. We will discuss further details later and summarize our key contributions as follows:
\begin{itemize}
\item We clarify that the WL test cannot preserve inter-node connection information, and we demonstrate that the metric based on the WL test is coarse.
\item We introduce the WL subtree as structural information in the neighborhood of a node, which enables us to define the tree edit distance between nodes. To compare WL subtrees, we use $L_1$-approximated tree edit distance ($L_1$-TED) in this paper.
\item We design a tree hash function and ensure that the probability of hash collision is theoretically low. Additionally, we propose a fast algorithm for computing $L_1$-TED using this tree hash function.
\item We propose a new fine-grained graph metric, Wasserstein Weisfeiler-Lehman Subtree (WWLS) distance, which can numerically represent slight structural differences.
\end{itemize}

\section{Preliminaries}
Bold typeface lower-case and upper-case letters such as $\mathbf{x}$ and $\mathbf{X}$ respectively denote a vector and a matrix. $\mathbf{x}_i$ denotes the $i$-th element of $\mathbf{x}$, $\mathbf{X}_i$ denotes the $i$-th row vector of $\mathbf{X}$, and $\mathbf{X}_{i,j}$ denotes the element at $(i,j)$ of $\mathbf{X}$. $\mathbb{R}^n_+$ denotes the space of nonnegative $n$-dimensional vectors, and $\mathbb{R}^{m\times n}_+$ denotes the space of nonnegative $m\times n$ size matrices. $\Delta_n$ denotes the probability simplex with $n$ bins. $\delta_x$ denotes the delta function at $x$, and $\delta(\cdot, \cdot)$ denotes the Kronecker delta. $\mathds{1}_n$ denotes an $n$-dimensional all-ones vector: $\begin{pmatrix} 1,\dots, 1\end{pmatrix}^{\rm T} \in \mathbb{R}^{n}$. $\lbrace \dots \rbrace$ denotes the set that does not allow duplication of elements, and $\ldblbrace \dots \rdblbrace$ denotes the multiset that allows elements to be repeated. $\mathcal{A} = \lbrace a_1, \dots, a_n \rbrace = \{ a_i \}_{i=1}^n$ denotes a set $\mathcal{A}$ consisting of $a_i$. {$\mathbb{F}[x_1, \dots, x_n]$ denotes a polynomial ring formed from the set of polynomials in $n$ variables over a field $\mathbb{F}$.} {$\mathbb{Z}/m\mathbb{Z}$ denotes a ring of integers modulo $m$, where $m \in \mathbb{Z}$ and $m \geq 2$.} $\mathbb{N}_+$ denotes the set of natural numbers starting from 1, and we define $\mathbb{N}_0 = \mathbb{N}_+ \cup \{0\}$. The graph data structure consists of a set of nodes ${V}$ and a set of edges $E \subseteq V^2$, which we write ${G}(V, E)$ or simply as $G$. In this paper, we consider only {\it undirected} graphs. $|V|$ denotes the number of nodes. $\mathcal{N}_{{G}}(v) = \{ u\in {V} \mid (v,u)\in {E}\}$ denotes the adjacent nodes of $v$ in $G$. {${\rm deg}(v)$ denotes the degree of node $v$. Node $v$ might also have a categorical label, which we write $\ell(v) \in \mathbb{N}_+$. $T$ denotes a tree. In particular, it refers to a rooted unordered {\it WL subtree} herein. For a tree $T$ with the root node of $v$, we express it as $T(v)$. $\mathcal{V}(T)$, $\mathcal{E}(T)$, and $\mathcal{L}(T)$ respectively denote the set of nodes, edges, and leaves of $T$. ${\rm dep}_T(v)$ denotes the depth of node $v$ in $T$. $T_1 \simeq T_2$ represents an isomorphism between $T_1$ and $T_2$. A non-root node $v\in \mathcal{V}(T)$ has a {\it parent}, written as {\rm parent}($v$). A non-leaf node $v\in \mathcal{V}(T)$ has $n$ children, written as $\mathcal{C}(v) = \lbrace c_i(v)\rbrace_{i=1}^n$, where $c_i(v)$ is the $i$-th child of $v$. A subtree $T'$ of $T$ is {\it complete} if, for node $v \in \mathcal{V}(T)$, {\rm parent}($v$) implies $v\in \mathcal{V}(T')$. We write $t$ for such a complete subtree. In addition, for complete subtree $t$ whose root node is $v$, we write $t(v)$. For other notations about $T$, we use the same method for $t$.

\section{Related Work}

\paragraph{Wasserstein distance.}
The Wasserstein distance is derived from the optimal transport (OT) problem, which attempts to determine the minimum transport cost by finding an optimal transportation plan between two probability distributions. The {\it discrete case} is defined as follows.

Let $\Delta_m = \{ \mathbf{a} \in \mathbb{R}^m_+ \mid \sum_{i=1}^m \mathbf{a}_i = 1\}$ and $\Delta_n = \{ \mathbf{b} \in \mathbb{R}^n_+ \mid \sum_{j=1}^n \mathbf{b}_j = 1\}$ denote two simplexes of the histogram with $m$ and $n$ in the same matrix space. Their respective probability measures are $\alpha = \sum_{i=1}^m \mathbf{a}_i {\delta_{x_i}}$ and $\beta = \sum_{i=1}^n \mathbf{b}_j {\delta_{y_j}}$. $\mathbf{C}\in \mathbb{R}^{m\times n}_+$ is a distance matrix, where $\mathbf{C}_{i,j}$ signifies the transportation cost (ground distance) between bin $i$ and bin $j$. $\mathbf{P} \in \mathbb{R}^{m \times n}_+$ is a transportation matrix, where $\mathbf{P}_{i,j}$ describes the amount of mass flowing from bin $i$ to bin $j$. The minimum total transportation cost between $\alpha$ and $\beta$, known as the {\it Wasserstein distance} associated with $\mathbf{C}$, is defined as
\begin{equation}
	\mathcal{W}(\alpha, \beta) = \underset{\mathbf{P}\in \mathbf{U}(\mathbf{a},\mathbf{b})}{\rm min} \sum_{i=1}^m \sum_{j=1}^n \mathbf{C}_{i,j}\mathbf{P}_{i,j},
	\label{eq:wass}
\end{equation}
where $\mathbf{U}(\mathbf{a},\mathbf{b}) = \{ \mathbf{P} \in \mathbb{R}^{m \times n}_+ \mid \mathbf{P}\mathds{1}_n = \mathbf{a} \ {\rm and} \ \mathbf{P}^{\rm T}\mathds{1}_m = \mathbf{b} \}$. The EMD~\citep{bonneel2011displacement} and Sinkhorn's algorithm~\cite{cuturi2013sinkhorn} are well-known methods that can solve the problem empirically with $\mathcal{O}(n^2)$ when $m=\mathcal{O}(n)$.

\paragraph{Weisfeiler-Lehman (WL) test and its kernel and distance forms.}
The graph isomorphism problem is an NP intermediate problem for determining whether two finite graphs are isomorphic~\citep{babai2016graph}. The WL test is an approximate solution to the problem that runs in linear time with respect to the size of the graph. It involves the aggregation of the node labels and their adjacent nodes to generate ordered strings, which are then hashed to generate new node labels. As the number of iterations increases, these labels will represent a larger neighborhood of nodes, allowing more extensive substructures to be compared~\cite{Togninalli19}. The WL test follows a recursive scheme, updating each node label multiple times. Given ${G}({V},{E})$, let $\ell^{(k)}(v)$ be the node label of $v\in V$ at the $k$-th iteration of the WL test. In particular, $\ell^{(0)}(v)$ is the original node label. Then the update formula for each node is
\begin{equation}
	\ell^{(k+1)}(v) = {\rm HASH}\left(\ell^{(k)}(v), \ldblbrace \ell^{(k)}(u) \mid u \in \mathcal{N}_{G}(v) \rdblbrace \right),
\label{eq:WL}
\end{equation}
where ${\rm HASH}(\cdot, \cdot)$ is the perfect hash that returns an integer value. The WL subtree kernel {\it a.k.a.} WL kernel is defined as the similarity of two graphs in terms of the inner product of the graph feature vectors as follows:
\begin{equation}
    K_{\rm WL}({G}, {G}') = \varphi({G}, \Sigma_h)^{\rm T} \varphi({G}', \Sigma_h),
    \label{eq:WLK}
\end{equation}
where $\varphi(\cdot, \cdot)$ is the graph feature vector, and $\Sigma_h = \{\ell^{(0)}, \ell^{(1)}, \dots \}$ is the set of all types of node labels that appear in ${G}$ and ${G}'$ with $h$ iterations of the WL test. $\varphi({G}, \Sigma_h)$ is  specifically defined as $\left(C_{\rm WL}({G}, \ell^{(0)}), \dots, C_{\rm WL}({G}, \ell^{(|{\Sigma}_h|)}) \right)\in \mathbb{N}_0^{|\Sigma_h|+1}$, where $C_{\rm WL}({G}, \ell^{(i)})$ is the function that returns the number of the occurrences of $\ell^{(i)}$ in ${G}$. The Wasserstein WL (WWL) distance, as its name implies, is a graph metric that combines the Wasserstein distance and the WL test. By applying the WL test $h$ times to node $v$, we obtain a sequence of $h+1$ different node labels that contains the original node label: $\mathcal{F}(v) = \left(\ell^{(0)}(v), \cdots, \ell^{(h)}(v) \right) \in \mathbb{N}_+^{h+1}$. It is called the {\it WL feature} of node $v$. The categorical case of the WWL is computed by the following optimization problem:
\begin{equation}
	\underset{\mathbf{P}\in \mathbf{U}(\mathbf{a},\mathbf{b})}{\rm{min}} \sum_{i=1}^{|V|} \sum_{j=1}^{|V'|} Ham \left(\mathcal{F}(v_i), \mathcal{F}(v_j)\right) \mathbf{P}_{i,j},
\end{equation}
where $Ham(\cdot, \cdot)$ is the normalized Hamming distance between two WL features.
\begin{figure}[!t]
\centering
\includegraphics[width=\linewidth]{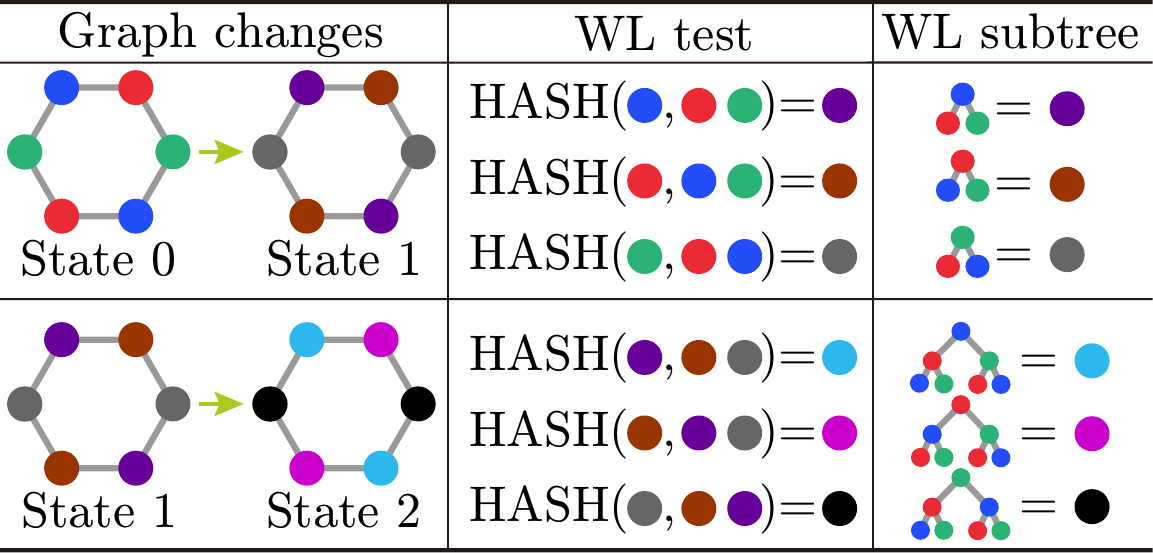}
\caption{Relationship between the WL test and the WL subtree. State 0 is the initial graph. Applying the WL test to State 0 yields State 1. Repeating this process on State 1 yields State 2. Each row in the figure shows the output of the WL test and the corresponding WL subtree for that state.}
\label{fig:WLsubtree01}
\end{figure}

\section{Problems in WL Kernel and WWL Distance}

\paragraph{Structural properties of WL test.}

The new label generated by the WL test is an integer hash value corresponding to the newly constructed tree, also known as the {\it WL subtree}. This tree is a rooted unordered tree with the following properties: (i) the root of the tree is the target node of the WL test, (ii) the tree is height-balanced, and (iii) the depth of each leaf is equal to the number of iterations. Figure~\ref{fig:WLsubtree01} illustrates the relationship between the WL test and the corresponding WL subtree. Note that the WL subtree contains inter-node connection information, which consists of the link information between the target node and its neighborhood. By performing the WL test $h$ times for a node, the WL subtree captures the inter-node connection information of the subgraph within the $h$-hop radius from the node. However, this critical structural information is lost because the WL test compresses the WL subtree into an {\it integer} value.

\paragraph{Problem definition.}
This paragraph discusses the simplicity of the measures in the WL kernel and the WWL distance. While the WL kernel has been successful for graph classification tasks, the simplicity of the measure in Eq.~(\ref{eq:WLK}) limits its ability to measure graph similarity. We can rewrite Eq.~(\ref{eq:WLK}) as $K_{\rm WL}({G}, {G}') = \sum_{i=0}^{h} \sum_{v\in {V}} \sum_{v'\in {V}'} \delta(\ell^{(i)}(v), \ell^{(i)}(v'))$. This equation shows that the WL kernel evaluates the node similarity score as either 1 or 0. Since each type of node label represents one type of WL subtree,  the WL kernel only judges the consistency of WL subtrees. This measure is suitable for graph isomorphism problems because they aim to determine whether two graphs are isomorphic or not, and this can be accomplished through binary judgments of 1 (isomorphic) or 0 (non-isomorphic). However, there are problems when measuring graph similarity. We can consider the following two situations. (i) First, we consider two nodes with the same neighborhood structure. If these two nodes have the same label, then the similarity is 1; otherwise, it is 0. In other words, if the labels do not match, the similarity is 0, regardless of how similar the neighborhoods are. (ii) Next, we consider two nodes with the same label but different neighborhood structures. In this case, the similarity is also 0. This extreme measure is not friendly to quantification, so the WL kernel does not measure fine-grained similarity at the node and graph levels. The WWL distance, on the other hand, uses a more advanced measure called the Wasserstein distance to improve the measurement capability at the graph level. However, for the categorical embedding of the WWL distance, its node-level measure remains a problem. The WWL distance takes the WL feature as a node feature and uses the normalized Hamming distance to define the ground distance. The dimension of the WL feature is $h+1$ if one runs the WL test $h$ times. It is noteworthy that a property of the WL test is that if two labels differ at iteration $k_0$ (where $k_0 \geq0$), then labels obtained by subsequent updates at iteration $k'>k_0$ are also different. Therefore, the Hamming distance between two WL features can only take at most $h+1$ different values with $h$ iterations. Combined with the fact that $h$ usually takes small values, it cannot capture the similarity between nodes with different starting labels and similar neighborhoods. Furthermore, since the practical effect of OT depends on the ground distance, pairwise matching of two graphs may not work well.% in WWL distance.

\section{Proposed Method}
\subsection{Tree Edit Distance between WL Subtrees}
Instead of using node labels to define the inter-node metric as in the related methods of the WL test, we use the WL subtree to compute the distance between tree structures. Given two nodes, $v$ and $v'$, in different graphs, we define the metric between the WL subtrees of $v$ and $v'$ using the tree edit distance (TED). The TED between unordered trees is a MAX SNP-hard problem, and this class of problems has constant-factor approximation algorithms but no approximation schemes unless P=NP. Therefore, it usually requires high time complexity. To solve this problem, we use an $L_1$-approximated TED ($L_1$-TED)~\citep{garofalakis2005xml, fukagawa2009constant}.

\begin{figure}[t]
	\centering
	\includegraphics[width=\linewidth]{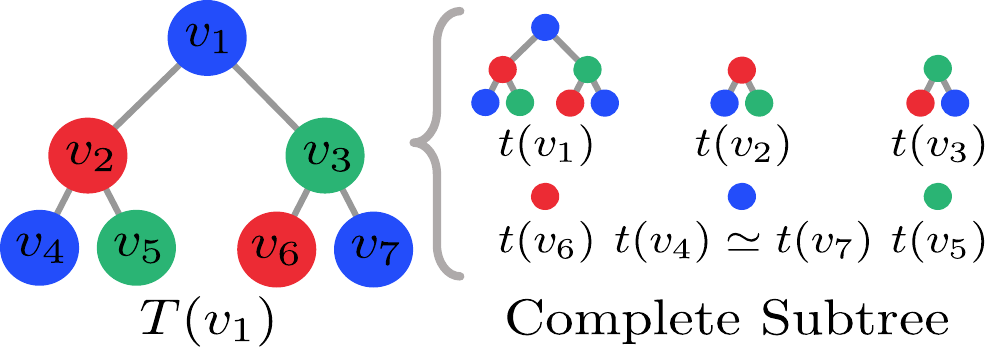}
	\caption{Illustration of the notation used for the WL subtree and complete subtree. $v_1$ is the blue node of the initial graph in Figure~\ref{fig:WLsubtree01}. We obtain $T(v_1)$ by 2 iterations of the WL test.}
	\label{fig:WLsubtree02}
\end{figure}

Before formally introducing our proposed algorithm, we introduce several necessary notations using Figure~\ref{fig:WLsubtree02}. $v_1$ is the blue node in Figure~\ref{fig:WLsubtree01}. By performing the WL test twice, we obtain a WL subtree rooted at $v_1$, which we designate as $T(v_1)$. $T(v_1)$ has seven nodes: $v_1, v_2, \dots, v_7$. Among them, $v_1, v_4$, and $v_7$ are fundamentally the same, but we treat all nodes differently. We denote the node-set of $T(v_1)$ as $\mathcal{V}(T(v_1))$. There are seven complete subtrees in $T(v_1)$: $t(v_1), t(v_2), \dots, t(v_7)$. Since $t(v_4)\simeq t(v_7)$, we regard them as identical. Thus, there are complete subtrees of six types.

\begin{figure}[t]
	\centering
 	\includegraphics[width=\linewidth]{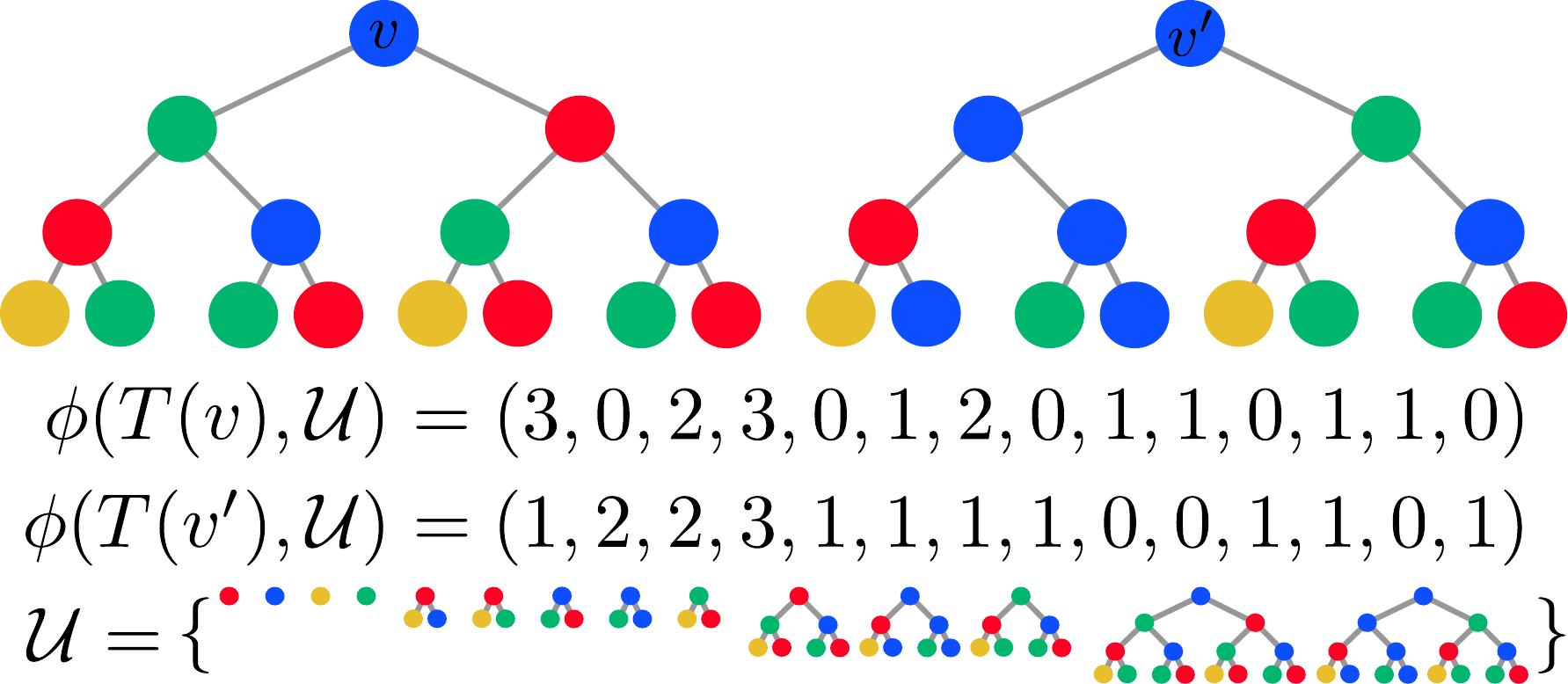}
	\caption{Illustration of the node embedding function. $v$ and $v'$ are the root nodes of two different WL subtrees, denoted as $T(v)$ and $T(v')$, respectively. Their feature vectors are $\phi(T(v), \mathcal{U})$ and $\phi(T(v'), \mathcal{U})$, respectively. $\mathcal{U}$ denotes the set of all types of complete subtrees for $T(v)$ and $T(v')$.}
	\label{fig:WLsubtree03}
\end{figure}

\paragraph{$\boldsymbol{L_1}$\textbf{-Approximated TED between WL subtrees.}}
To compare two nodes, $v$ and $v'$, we first construct their WL subtrees. Next, we compute the $L_1$ norm of the difference between the node feature vectors of $T(v)$ and $T(v')$. These operations are defined by the distance function $d_\phi$: $\mathcal{T}\times \mathcal{T} \rightarrow \mathbb{Z}/M\mathbb{Z}$, where $\mathcal{T}$ is the set of all types of WL subtrees and $M$ is a prime number. Formally, we define
\begin{equation}
	d_{\phi} \left(T(v), T(v')\right) = || \phi(T(v), \mathcal{U}) - \phi(T(v'), \mathcal{U}) ||_1,
	\label{eq:ted}
\end{equation}
where $\phi(\cdot, \cdot)$ is the feature vector of the corresponding tree. $\mathcal{U} = \lbrace t_1, t_2, \dots \rbrace$ is the set of all types of complete subtrees of $T(v)$ and $T(v')$, and any two complete subtrees $t_i, t_j\in \mathcal{U}$ are $t_i \not \simeq t_j$ for $i\neq j$. We define $\phi(T, \mathcal{U})$ as $\left(C_t(T, t_1), \dots, C_t(T, t_{|\mathcal{U}|})\right) \in \mathbb{N}_0^{|\mathcal{U}|}$, where $C_t(T, t_i)$ is a function that returns the number of occurrences of $t_i$ in $T$. Figure~\ref{fig:WLsubtree03} presents an intuitive description of $\phi(\cdot, \cdot)$. Using the properties proved by~\citep{fukagawa2009constant}, the true TED $d_{\rm TED}(\cdot, \cdot)$ between $T(v)$ and $T(v')$ can be bounded as follows:
\begin{equation}
	\frac{d_{\phi}(T(v), T(v'))}{2h+2} \leq d_{\rm TED}(T(v), T(v')) \leq d_{\phi}(T(v), T(v')),
\label{eq:bound}
\end{equation}
where $h$ denotes the height of $T(v)$ and $T(v')$. It is noteworthy that the height of the WL subtree is equal to the number of iterations. This inequality implies that a smaller $h$ has closer $d_{\phi}(\cdot, \cdot)$ to $d_{\rm TED}(\cdot, \cdot)$.
\begin{figure}[t]
	\centering
 	\includegraphics[width=\linewidth]{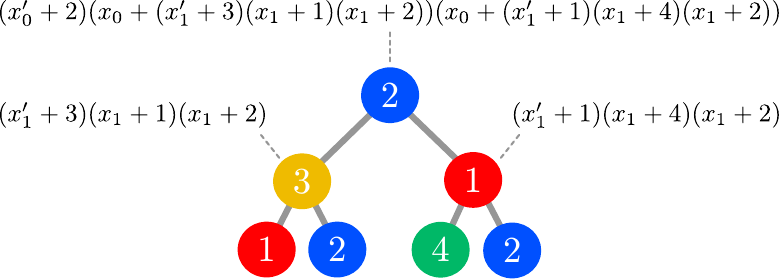}
	\caption{Illustration of Eq~(\ref{eq:hash}). Red, blue, yellow, and green nodes are labeled as 1, 2, 3, and 4, respectively. Variables at depth 0 are denoted as $x_0$ and $x_0'$, and variables at depth 1 are denoted as $x_1$ and $x_1'$. The polynomial represents the complete subtree with non-leaf nodes as roots.}
	\label{fig:WLsubtree04R}
\end{figure}

\paragraph{A fast algorithm for \boldsymbol{$L_1$}-TED.}
To compute $\phi(T(v), \mathcal{U})$ and $\phi(T(v'), \mathcal{U})$, one must know all types of complete subtrees that appear in $T(v)$ and $T(v')$. However, enumerating all types of complete subtrees and searching each one from the WL subtree requires a high computational cost. Therefore, we propose an efficient algorithm that involves designing a hash function, mapping each complete subtree to an integer value during a post-order depth-first search (DFS) of the WL subtree. Lemma~\ref{lemma:sz} provides an upper bound on the probability that the values of two multivariate polynomials agree under its given conditions. Our idea is to assign a polynomial to each complete subtree. If we can demonstrate that the collision probability between polynomial values can be significantly reduced to a negligible level, then any polynomial value can serve as a hash value.

Given a WL subtree $T$ with height $h$, we assume that there are two variables for each depth except depth $h$ in $T$: {$x_{i}$, $x_{i}'\in \mathbb{Z}/M\mathbb{Z}$ for depth $i\in \lbrace 0, \dots, h-1\rbrace$.} We do not set variables at depth $h$ because the leaves themselves are the simplest complete subtrees, and their node labels already represent the type of simplest complete subtrees. To simplify the expression, we denote $p(v): V \rightarrow \mathbb{Z}/M\mathbb{Z}[x_0, x_0', \dots, x_{h-1}, x_{h-1}']$ as the polynomial of $t(v)$ and $c_i(v)$ as the $i$-th child of $v$. To compute the polynomials dynamically, we traverse the WL subtree in post-order DFS. The polynomial $p(v)$ is defined as follows:
\begin{equation}
p(v) = \left(x'_{{\rm dep}_T(v)} + \ell^{(0)}(v) \right) \prod_{i=1}^n \left(x_{{\rm dep}_T(v)} + p(c_i(v))\right),
\label{eq:hash}
\end{equation}
where $n$ denotes the number of children of $v$. It is important to perform modulo $M$ at each intermediate step of Eq.~(\ref{eq:hash}) to prevent overflow. Figure~$\ref{fig:WLsubtree04R}$ is an illustration of the algorithm. For $v\notin {\rm leaf}(T)$, the polynomial $p(v)$ corresponding to the $t(v)$ has variables $\left \{ x_i, x_i'\mid i \in \{{\rm dep}_T(v), \dots, h-1\} \right \}$, and its degree is $|\mathcal{L}(t(v))|$.

\paragraph{Theoretical guarantees for the tree hash function.} We use multiplication between $x_{{\rm dep}_{T}(v)} + p(c_i)$ for $i \in \lbrace 1, \dots, n \rbrace$ to ensure that the same polynomial can be obtained even if the order of children of $v$ is different. Furthermore, to consider the information of $v$ itself, we multiply by $x_{{\rm dep}_T(v)}'+\ell^{(0)}(v)$, where we use $x_{{\rm dep}_T(v)}'$ instead of $x_{{\rm dep}_T(v)}$ to distinguish $v$ from its children. We also provide theoretical guarantees for our algorithm. Proposition~\ref{pro:01} shows that the polynomial constructed in this way has a one-to-one correspondence with a complete subtree.

\begin{proposition}
Two complete subtrees $t(v_1)$ and $t(v_2)$ are isomorphic if and only if  polynomials $p(v_1)$ and $p(v_2)$ agree. 
\label{pro:01}
\end{proposition}

If Proposition~\ref{pro:01} holds, then Proposition~\ref{pro:02} gives an upper bound on the collision probability between the integer values of two polynomials.

\begin{proposition}
Let $p(v_1)$ and $p(v_2)$ be polynomials corresponding to two complete subtrees $t(v_1)$ and $t(v_2)$, respectively. Then, the upper bound of the collision probability between integer values of $p(v_1)$ and $p(v_2)$ is $(|\mathcal{L}(t(v_1))|+|\mathcal{L}(t(v_2))|)/M$. 
\label{pro:02}
\end{proposition}

According to Proposition~\ref{pro:02}, choosing a sufficiently large prime number $M$ can reduce the collision probability between the inter values of two polynomials to a low enough level. Typically, we choose $M=10^9+7$ for 64-bit computers. To prevent hash collisions when $N$ is large, we assign $k$ hash values to one complete subtree. Proposition~\ref{pro:bound} provides lower and upper bounds on the probability of at least one collision in generating $N$ hashes, denoted as ${\rm Pr}_{_{\rm Hash}}(N)$.

\begin{proposition}
Let $\xi$ be the maximum number of leaves in all complete subtrees. Then, ${\rm Pr}_{_{\rm Hash}}(N)$ is bounded by
$$1-e^{\frac{-N(N-1)}{2M^k}} \leq {\rm Pr}_{_{\rm Hash}}(N) \leq 1 - \left(1-\left(\frac{2\xi}{M}\right)^k \right)^{\frac{N(N-1)}{2}}.$$
\label{pro:bound}
\end{proposition}

The proofs for the above three propositions can be found in \textbf{Proof of Propositions} section.

\subsection{Wasserstein Distance between Graphs}
We propose a novel graph metric that combines the $L_1$-TED and OT to measure slight differences in structure by reflecting $L_1$-TED at the graph level. First, we consider the set $\mathcal{U}^*$ of all complete subtrees obtained from two given graphs $G$ and $G'$, and then embed them into the same metric space $(\mathbb{N}_0^{|\mathcal{U}^*|}, d_{\phi})$ by computing $\phi(\cdot)$ for all nodes of $G$ and $G'$. Each node of the two graphs is embedded respectively at points $x_1, \dots, x_{|V|}$ and $y_1, \dots, y_{|V'|}$. We define two histograms of $\mathbf{a}$ and $\mathbf{b}$ in the probability simplices $\Delta_{|V|}$ and $\Delta_{|V'|}$, respectively, to serve as weights for each point. We define $\mu({G}) =\sum_{i=1}^{|V|} \mathbf{a}_i {\delta_{x_i}}$ as a discrete measure $\mu({G})$ with weights $\mathbf{a}$ on the locations $x_1,\dots, x_{|V|}$. Similarly, we define $\mu({G}')=\sum_{j=1}^{|V'|}\mathbf{b}_j {\delta_{y_j}}$. Using the above, we can define the Wasserstein distance between $\mu({G})$ and $\mu({G}')$ as follows:
\begin{equation}
	\mathcal{W}\left(\mu({G}), \mu({G}')\right) =\!\! \underset{\mathbf{P}\in \mathbf{U}(\mathbf{a},\mathbf{b})}{\rm min} \! \sum_{i=1}^{|V|} \sum_{j=1}^{|V'|} d_{\phi}\!\left(T(v_i), T(v_j)\right) \!\mathbf{P}_{i,j}.
	\label{eq:wass}
\end{equation}
We call this the {\it Wasserstein Weisfeiler-Lehman Subtree (WWLS) distance}. The computation procedure is summarized in Algorithm~\ref{alg:2}.

\begin{algorithm}[t]
	\caption{Enumerating all types of complete subtrees: ${\rm DFS}_{\rm WL}(\mathcal{G}, h, v, d)$.}
	\label{alg:1}
 	\begin{algorithmic}
	\ENSURE{$\mathcal{P}(v) = \lbrace p(u) \mid u \in \mathcal{V}(T(v))\rbrace$.}
	\REQUIRE{the set of tuples consisting of a node and its adjacent nodes in $G$: $\mathcal{G}= \lbrace (u, \mathcal{N}_G(u)) \mid u\in V\rbrace$; number of iterations $h$; target node $v$; depth of WL subtree $d=0$.}
         \IF{$d > h$}
         	\STATE return.
         \ENDIF
         \STATE Add a new node $v$ to the WL subtree.
         \FOR{each $u$ in $\mathcal{N}_G(v)$} 
         \STATE ${\rm DFS}_{\rm WL}(\mathcal{G}, h, u, d+1)$.
         \ENDFOR
         \IF {$d \neq h$}
         \STATE $p(v)$ $\leftarrow$ Calculate by Eq.~(\ref{eq:hash}).
         \ENDIF
         \STATE Record the hash value $p(v)$ in $\mathcal{P}(v)$.
	\end{algorithmic} 
\end{algorithm}

\begin{algorithm}[t]
	\caption{Computing the WWLS distance.}
	\label{alg:2}
	\begin{algorithmic}
	\ENSURE{$\mathcal{W}(\mu(G), \mu(G'))$.}
	\REQUIRE{two graphs ${G}({V}, {E})$ and ${G}'({V}', {E}')$.}
    	\STATE $U = \bigcup_{v\in V} \mathcal{P}(v)$ $\leftarrow$ Calculate by Algorithm~$\ref{alg:1}$.
    	\STATE $U' = \bigcup_{v'\in V'} \mathcal{P}(v')$ $\leftarrow$ Calculate by Algorithm~$\ref{alg:1}$.
    	\STATE $\mathcal{U} = U \cup U'$.
        \STATE $\mathbf{C}\leftarrow d_{\phi}(\phi(T(v_i), \mathcal{U}), \phi(T(v_j), \mathcal{U}))$ for all combinations of $i$ and $j$.
        \STATE $\mathcal{W}(\mu(G), \mu(G'))$ $\leftarrow $ Calculate by Eq.~(\ref{eq:wass}).
        \STATE return $\mathcal{W}(\mu(G), \mu(G'))$.
	\end{algorithmic} 
\end{algorithm}

\section{Time Complexity Analysis}

\begin{table}
    	\caption{Results of the Runtime Experiments. The time required to compute the distance between all pairs of graphs for each dataset is shown below (in seconds). For MUTAG, PTC-MR, and ENZYMES, $h = 2$; for IMDB-B, $h=1$. The reason for setting such $h$ is explained in Experiment 3. The details of the dataset are summarized in Tables~\ref{table:gc0} and \ref{table:gc1}.}
	\begin{center}
	\begin{tabular}{ |l|r|r| } 
	\hline
    	& WWL & WWLS \\
	\hline
	MUTAG & 3.46 & 4.42 \\ \hline
	PTC-MR & 10.49 & 12.25 \\ \hline
	ENZYMES & 65.68 & 108.22 \\ \hline
	IMDB-B & 106.20 & 129.00 \\ 
    	\hline
	\end{tabular}
	\end{center}
	\label{table:runtime}
\end{table}

We summarize the above computation procedures in Algorithms~\ref{alg:1} and \ref{alg:2}. First, we analyze Algorithm~\ref{alg:1}. The construction of the WL subtree and the computation of the hash value can be implemented in a single DFS framework. For each WL subtree, Eq.~(\ref{eq:hash}) is executed $|\mathcal{V}(T) \backslash \mathcal{L}(T)|$ times. Therefore, the overall time complexity is $\mathcal{O}(|\mathcal{V}(T)|+|\mathcal{E}(T)|+n|\mathcal{V}(T) \backslash \mathcal{L}(T)|)$. Assuming that the average degree of the graph is $\bar{d}$, we  further consider the WL subtree to be an approximately perfect $\bar{d}$-ary tree. Then, $n=\bar{d}, |\mathcal{V}(T)| = \frac{1-\bar{d}^{(h+1)}}{1-\bar{d}}$, $|\mathcal{L}(T)|=\bar{d}^h$ and $|\mathcal{E}(T)| = |\mathcal{V}(T)|-1$. Finally, it takes $\mathcal{O}\left(\left(3\frac{\bar{d}^{h+1}-\bar{d}}{\bar{d}-1}+1\right)|V|\right)$ for one graph. This is {\it linear time complexity} with respect to the size of the graph and exponential time complexity with respect to $h$. Next, we analyze Algorithm~\ref{alg:2}. For convenience, assume that $|V'|= \mathcal{O}(|V|)$, $h$ is a constant, and $\tau$ is the average number of types of complete subtrees present in the two WL subtrees. We run Algorithm~\ref{alg:1} twice and then compute $\mathbf{C}$ using pairwise $L_1$-TED. Considering that the computation of the $L_1$ norm requires $\mathcal{O}(\tau)$, it takes $\mathcal{O}\left(2\left(3\frac{\bar{d}^{h+1}-\bar{d}}{\bar{d}-1}+1\right)|V|+\tau |V|^2\right)$ for these computations. In addition, computing the Wasserstein distance takes approximately quadratic time complexity. Therefore, the overall time complexity is $\mathcal{O}\left(2\left(3\frac{\bar{d}^{h+1}-\bar{d}}{\bar{d}-1}+1\right)|V|+ (\tau + 1 )|V|^2\right)$. To verify its real runtime efficiency, we conduct runtime experiments and show the results in Table~\ref{table:runtime}. The heavy processing parts of WWL and WWLS are written in C++, and we run programs on macOS Monterey, Intel(R) Core(TM) i5-7360U CPU @ 2.30GHz. As seen in Table~\ref{table:runtime}, although WWLS is slower than WWL, the difference is within acceptable limits.

\section{Experiments}
We conduct two types of experiments: metric validation and graph classification experiments. In the metric validation experiments, we demonstrate the effectiveness of the WWLS as a metric. In the graph classification experiments, we confirm the adaptability of the metric to graph classification, which represents one of its diverse applications. All experiments are conducted in the same environment as the runtime experiments. \textbf{Source code:} \underline{\url{https://github.com/Fzx-oss/WWLS}}.

\subsection{Metric Validation Experiments}
\paragraph{Experiment 1: Metric validation experiments.}
First, we evaluate WWLS on metric validation experiments. A good metric should be able to measure slight differences between two graphs. This experiment verifies this point. We randomly generate two graphs with 50 nodes and keep increasing the {\it edge noise} of one of them. We adopt two methods to add edge noise: one replaces the edge $a$--$b$ with $a$--$c$, and the other adds a new edge. We also prepare cycle and grid graphs as synthetic datasets. As baselines, we use WWL distance, Gromov-Wasserstein (GW) distance based on the shortest path length~\cite{peyre2016gromov}, and the Frobenius norm of the difference between Laplacian matrices. For the Laplacian matrix, we prepare two matrices: one aligned and one intentionally disordered with a substitution matrix. The ideal but impractical baseline is the one using an aligned Laplacian matrix. The graph alignment problem is another important issue in graph machine learning, allowing easy comparison of graph structures. We set the number of iterations as $h = 2$ for WWL and WWLS.

We observe changes in distance values with increasing noise and summarize the results in Figure~\ref{fig:distacne}. The Frobenius norm of the difference between aligned Laplacian matrices shows a smooth curve with monotonic growth, as shown by the green dashed line. Next, we examine the control group. The unaligned Laplacian matrices fail to measure the graph structure. WWLS succeeds in drawing a smooth curve close to the ideal case, whereas WWL and GW show a steep curve in which the values increase rapidly, even with small noise. Fluctuations in the distance values of GW are particularly noticeable because the addition of noise could easily change the length of the shortest path. We infer that both WWL and shortest-path-length-based GW are biased toward comparing the consistency of graphs.

\begin{figure*}
	\centering
	\includegraphics[width=\linewidth]{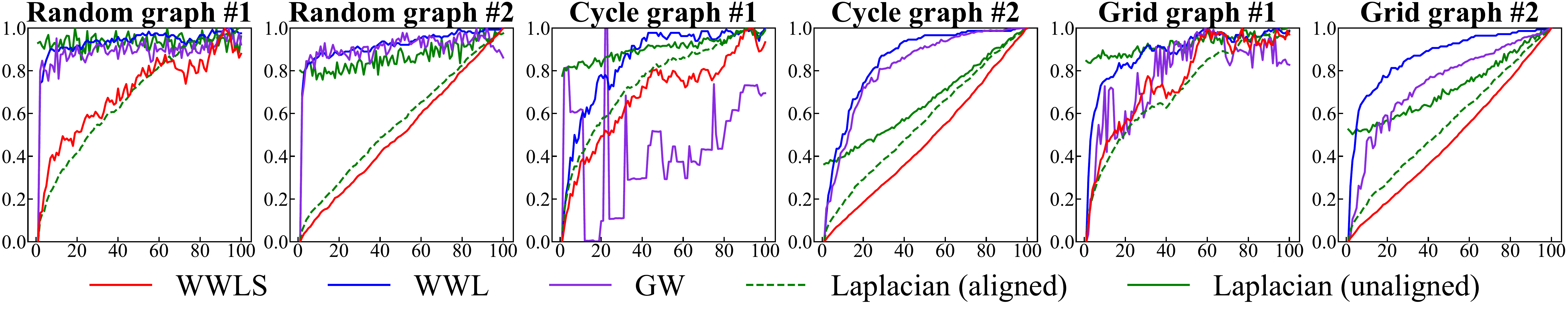}
	\caption{Results of Experiment 1. Change of the distance values with the increase of the edge noise. The vertical axis shows the distance value. The horizontal axis shows the number of times noise is added. {\#1} refers to the noise type that replaces edges. {\#2} refers to the noise type that adds edges. For ease of comparison, distance values are normalized by the maximum value.}
	\label{fig:distacne}
\end{figure*}

\begin{table*}
	\caption{Results of Experiment 2. Results are reported as mean $\pm$ standard deviation of ten repetitions. The best result for each dataset is marked in {bold}. ``--'' refers to results not reported in the original paper.}
	\renewrobustcmd{\bfseries}{\fontseries{b}\selectfont}
	\renewrobustcmd{\boldmath}{}
	\begin{center}
    	\begin{tabular}{lcccccccc} \toprule
              & MUTAG               & PTC-MR              & COX2                & ENZYMES              & PROTEINS            \\ 
		Graphs                & 188   & 344     & 467    & 600   & 1113  \\ 
		Classes               & 2       & 2         & 2        & 6     & 2     \\ 
		Avg. Nodes         & 17.93 & 14.29 & 41.22 & 32.63 & 39.06 \\ 
		Avg. Degree        & 1.10  & 1.03    & 1.05   & 1.90  & 1.86  \\ 
        \midrule \midrule
        WL            & 85.75$\pm$1.96       & 61.21$\pm$2.28       & 79.67$\pm$1.32         & 54.27$\pm$0.94       & 73.06$\pm$0.47      \\
        WL-OA      & 86.10$\pm$1.95       & 63.60$\pm$1.50       & 81.08$\pm$0.89         & 58.88$\pm$0.85       & 73.50$\pm$0.87       \\
        WL-PM      & 87.77$\pm$0.81       & 61.41$\pm$0.81       &  --                               & 55.55$\pm$0.56       & --                              \\
        WWL         & 87.27$\pm$1.50       & 66.31$\pm$1.21        & 78.29$\pm$0.47        & 59.13$\pm$0.80       & 74.28$\pm$0.56       \\ 
    	GIN           & 84.51$\pm$1.56       & 56.20$\pm$2.18        & \bf{82.08$\pm$0.93}  & 39.35$\pm$1.53       & 71.93$\pm$0.63      \\ \midrule
    	WWLS      & \bf{88.30$\pm$1.23} & \bf{67.32$\pm$1.09} & 81.58$\pm$0.91        & \bf{63.35$\pm$1.14}  & \bf{75.35$\pm$0.74}  \\ \bottomrule
	\end{tabular}
	\end{center}
	\label{table:gc0}
\end{table*}

\begin{table}
	\caption{Results of Experiment 2. Reported in the same manner as in Table~\ref{table:gc0}. } 
	\renewrobustcmd{\bfseries}{\fontseries{b}\selectfont}
	\renewrobustcmd{\boldmath}{}
	\begin{center}
	\begin{tabular}{lccccc} \toprule 
          			& NCI1                 & BZR      &  IMDB-B           & IMDB-M             & COLLAB              \\ 
		Graphs    & 4110  & 405   & 1000   & 1500  & 5000\\
		Classes   & 2     & 2   & 2      & 3     & 3\\
		Avg. Nodes  & 29.87 & 35.75 & 19.77  & 13.00 & 74.49\\
		Avg. Degree & 1.08  & 1.07 & 4.88   & 5.07  & 32.99\\
        \midrule \midrule
        WL                & 85.76$\pm$0.22        & 87.16$\pm$0.97  &  71.15$\pm$0.47        & 50.25$\pm$0.72         & 79.02$\pm$1.77      \\
        WL-OA          & 85.95$\pm$0.23         & 87.43$\pm$0.81 &  74.01$\pm$0.66        & 49.95$\pm$0.46         & 80.18$\pm$0.25      \\
        WL-PM   	    & \bf{86.40$\pm$0.20}  & --       & -- & --  & --\\
        WWL            & 85.75$\pm$0.25       & 84.42$\pm$2.03  &  74.37$\pm$0.83        & --                                 & --                  \\
    	GIN              & 77.86$\pm$0.49       & 83.86$\pm$0.95   & 72.52$\pm$0.95         & 49.41$\pm$1.16         & 78.32$\pm$0.32      \\ \midrule
        WWLS         & 86.06$\pm$0.09       & \bf{88.02$\pm$0.61}  & \bf{75.08$\pm$0.31}   & \bf{51.61$\pm$0.62}  & \bf{82.81$\pm$0.16} \\ \bottomrule
	\end{tabular}
	\end{center}
	\label{table:gc1}
\end{table}

\subsection{Graph Classification Experiments}

\paragraph{Conversion from metric to graph kernel.}
In the context of the graph classification task, we adopt the approach of previous work~\cite{HUANG2021108281} and introduce the indefinite kernel with the following formula:
\begin{equation}
	K(G, G') = \exp \left({-\gamma \mathcal{W}\left(\mu(G), \mu(G')\right)}\right),
	\label{eq:kernel}
\end{equation}
where $\gamma$ is a parameter. It is not guaranteed that Eq.~(\ref{eq:kernel}) is symmetric positive semi-definite. Therefore, we adopt the Krein SVM~\citep{loosli2015learning}, which can solve SVM with kernels that are usually troublesome, such as large numbers of negative eigenvalues.

\paragraph{Experiment 2: General graph classification experiments.} \textbf{(i) Datasets.} We use TUD benchmark datasets, specifically selecting ten frequently used datasets, which can be grouped into two categories: (1) Bioinformatics datasets, including MUTAG, PTC-MR, COX2, ENZYMES, PROTEINS, NCI1, and BZR; and (2) Social network datasets, including IMDB-B, IMDB-M, and COLLAB. \textbf{(ii) Evaluation methods.} We employ a commonly used evaluation method for graph kernels~\citep{Morris+2020}, randomly splitting the data into a training set ($90\%$) and a test set ($10\%$), with a portion of the training set reserved for validation to tune the parameters. We repeat this evaluation ten times and report the average accuracy and standard deviation. \textbf{(iii) Baselines.} We compare our approach with five baselines: WL kernel, WL Optimal Assignment kernel (WL-OA)~\citep{kriege2016valid}, WL Pyramid Match kernel~(WL-PM)~\citep{nikolentzos2017matching}, WWL kernel, and Graph Isomorphism Network (GIN)~\citep{xu2018how}. All models are related methods to the WL test. For the WL-PM and WWL, we cite results from original papers. For the remaining graph kernels, we cite results from the survey paper~\cite{9307216}, and our evaluation method is consistent with theirs. Since the original paper of GIN does not set up a validation set, we use the same conditions as \cite{Morris+2020} to make a more fair comparison. The parameters of the WWLS are set as follows: we adjust the iteration number $h$ within $\{2, 3\}$ for the bioinformatics datasets and $h=1$ for the social datasets due to their large node degrees; we adjust the parameter $\gamma$ of Eq.~(\ref{eq:kernel}) within $\{ 10^{-4}, 10^{-3}, \dots, 10^{-1}\}$; and we adjust the regularization parameter $C$ of SVM within $\{10^{-3}, 10^{-2}, \dots, 10^3\}$. 

We present the results in Tables~\ref{table:gc0} and \ref{table:gc1}. As results show, the WWLS outperforms the baselines on all datasets except COX2 and NCI1 but is in the second position. The second experiment demonstrates the effectiveness of the WWLS on the graph classification tasks.

\begin{figure*}
	\centering
	\includegraphics[width=\linewidth]{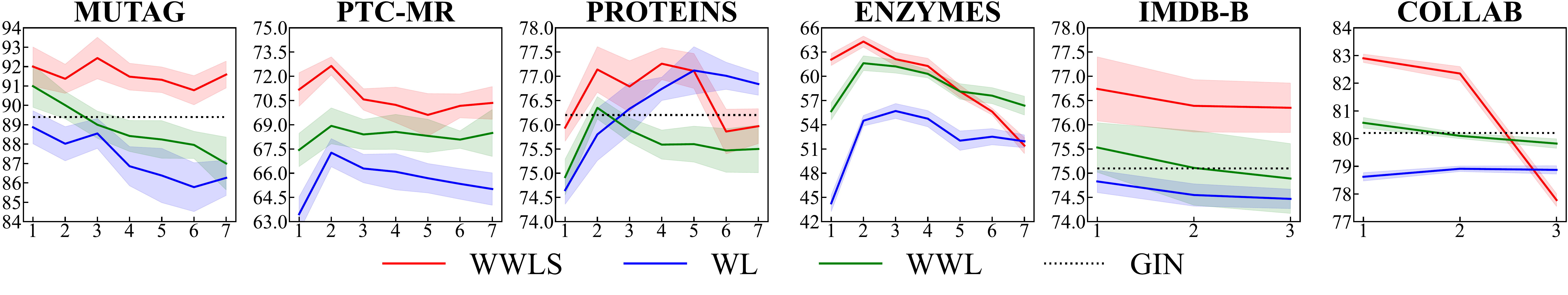}
	\caption{Results of Experiment 3. Classification accuracy with the number of iterations. The horizontal axis shows the number of iterations. The vertical axis shows the accuracy ($\%$). Shaded areas represent standard deviations (mean $\pm$ standard deviation). Datasets not used in the original paper of GIN are not shown with dashed lines.}
	\label{fig:potential}
\end{figure*}

\paragraph{Experiment 3: Maximum classification performance of models.}
Graph classification experiments based on 10 times 10-fold cross-validation typically have two ways of recording results: the {\it mean $\pm$ standard deviation of 10 repetitions}, as employed in the experiment above; and {\it mean $\pm$ standard deviation of 100 runs} (10 repetitions with 10 folds). According to \citep{Morris+2020}, the former usually has a low standard deviation, whereas the latter has a high standard deviation. This is due to significant variation in the graph datasets and, therefore, the parameters selected in the validation set that do not work in the test. Thus, the performance of models cannot be sufficiently differentiated in the second experiment. The third experiment evaluates the maximum performance of the WL, WWL, and WWLS. We perform 10 times 10-fold experiments in each iteration and search all parameters of models and SVM. The one with the highest accuracy is regarded as the maximum performance in that iteration. Since this evaluation is very similar to the one adopted by GIN, we cite the best results of the original paper as an additional baseline. We enumerate $h$ with $\{1, 2, \dots, 7\}$, and for $\gamma$ of the WWL and WWLS, we adjust them within $\{10^{-3}, 10^{-2}, \dots, 1\}$. For $C$ of SVM, we adjust it within $\{10^{-3}, 10^{-2}, \dots, 10^3\}$. 

We summarize the results in Figure~\ref{fig:potential}, which provide a deeper understanding and insight into the performance differences among the three methods. We observe that the WWLS significantly outperforms the WL, WWL, and GIN when the iteration is small, such as $h=2$ or 3 for the bioinformatics datasets and $h=1$ for the social network datasets. As $h$ increases, all kernel methods exhibit an overall decreasing trend in accuracy, indicating the importance of local structure. In terms of the trend of decreasing accuracy, WWLS shows the most significant decrease, which can be attributed to a weakening of the bounds in Eq.~(\ref{eq:bound}) as $h$ increases. For these reasons, we actually set $h$ to a small number, such as 1--3.

\section{Proof of Propositions}

\subsection{Proof of Proposition 1}

\begin{lemma}
Let ${T}(v_1)$ and ${T}(v_2)$ be general finite rooted trees with roots $v_1$ and $v_2$, respectively. Then ${T}(v_1)$ is isomorphic to ${T}(v_2)$ if and only if there exists a bijection $f: \mathcal{C}(v_1) \rightarrow \mathcal{C}(v_2)$ such that for all $v\in \mathcal{C}(v_1)$ the subtree rooted at $v$ is isomorphic to the subtree rooted at $f(v)$ and $v_1 \simeq v_2$.
\label{lemma:0-1}
\end{lemma}
\noindent
Lemma \ref{lemma:0-1} can be obtained directly from the definition of tree isomorphism~\cite{buss1997alogtime}.

\begin{lemma}[Expansion of Gauss' Theorem~\cite{Bosch2018}]
The polynomial ring $\mathbbmsl{F}[x_1, \dots, x_n]$ is a unique factorization domain (UFD) if and only if $\mathbbmsl{F}$ is a UFD.
\label{lemma:0-2}
\end{lemma}

\noindent We prove Proposition 1 using Lemmas~\ref{lemma:0-1} and~\ref{lemma:0-2}.

\begin{proof}
We use mathematical induction to prove two directions. For both proofs, we assume that there are two complete subtrees $t(v_1)$ and $t(v_2)$ with height $h_1$ and $h_2$, respectively. We also define notations for the proof. Let $l_i(t(v_1)) \in \mathcal{L}(t(v_1))$ and $l_j(t(v_2)) \in \mathcal{L}(t(v_2))$ are the $i$-th and $j$-th leaves of $t(v_1)$ and $t(v_2)$, respectively.

{\it Necessity} [$t(v_1) \simeq t(v_2) \rightarrow p(v_1) = p(v_2)$]. Since $t(v_1) \simeq t(v_2)$, the heights of the two trees are equal. First, we consider the case of nodes with a height of zero. Assuming that $l_i(v_1)$ and $l_j(v_2)$ have a correspondence, then $p(l_i(v_1)) = p(l_j(v_2))$ is true. For nodes with a height $h_1-1$, let us assume that $t(c_i(v_1))) \simeq t(c_j(v_2)) \rightarrow p(c_i(v_1)) = p(c_j(v_2))$ holds. By Lemma~\ref{lemma:0-1}, we can obtain $p(c_i(v_1)) = p(f(c_i(v_1)))$. Next, we consider when the height is $h_1$.
\begin{eqnarray*}
&& p(v_1) \\
=\!\!\!\!\!\!\!\!\!&& (x'_h + \ell^{(0)}(v_1)) (x_h + p(c_1(v_1))) \cdots (x_h + p(c_m(v_1)))\\
=\!\!\!\!\!\!\!\!\!&& (x'_h + \ell^{(0)}(v_1)) (x_h + p(f(c_1(v_1)))) \cdots (x_h + p(f(c_m(v_1)))) \\
=\!\!\!\!\!\!\!\!\!&& (x'_h + \ell^{(0)}(v_1)) (x_h + p(c_1(v_2))) \cdots (x_h + p(c_m(v_2)))\\
=\!\!\!\!\!\!\!\!&& (x'_h + \ell^{(0)}(v_2)) (x_h + p(c_1(v_2))) \cdots (x_h + p(c_m(v_2)))  \\
=\!\!\!\!\!\!\!\!&& p(v_2).
\end{eqnarray*}
We have demonstrated that if the statement holds true for the case where the height is $h_1-1$, then it necessarily holds true for the subsequent case where the height is $h_1$.

{\it Sufficiency} [$p(v_1) = p(v_2) \rightarrow t(v_1)\simeq t(v_2)$]. If $t(v_1)$ and $t(v_2)$ have different heights, then $p(v_1) \neq p(v_2)$ because they have different variables. Therefore, complete subtrees have the same height if their polynomials agree. For nodes with a height of zero, it is evident that $p(l_i(t(v_1)))=p(l_j(t(v_2)) \rightarrow l_i(t(v_1)) \simeq l_j(t(v_2))$. Moving on to nodes with a height $h_1-1$, let us assume that $p(c_i(v_1)))=p(c_j(v_2)) \rightarrow t(c_i(v_1)) \simeq t(c_j(v_2))$. Now, let us consider the case that the height is $h_1$. We assume that $v_1$ and $v_2$ have $m$ and $n$ children, respectively. Since $p(v_1)=p(v_2)$, we can obtain
\begin{eqnarray*}
&& p(v_1) \\
=\!\!\!\!\!\!\!\!&&  (x'_h + \ell^{(0)}(v_1)) (x_h + p(c_1(v_1))) \cdots (x_h + p(c_m(v_1))) \\
=\!\!\!\!\!\!\!\!&& (x'_h + \ell^{(0)}(v_2)) (x_h + p(c_1(v_2))) \cdots (x_h + p(c_n(v_2))) \\
=\!\!\!\!\!\!\!\!&& p(v_2).
\end{eqnarray*}
The polynomials $p(v_1)$ and $p(v_2)$ belong to the ring $(\mathbb{Z}/M\mathbb{Z}[x_0, $ $x_0',\dots, x_{h-1}, x_{h-1}'])[x_h, x_h']$. Since $\mathbb{Z}/M\mathbb{Z}$ is a field, it is a unique factorization domain (UFD) by definition. By applying Lemma~\ref{lemma:0-2}, we can also conclude that $\mathbb{Z}/M\mathbb{Z}[x_0, x_0', \dots, x_{h-1}, x_{h-1}']$ is a UFD, which implies that $(\mathbb{Z}/M\mathbb{Z}[x_0, x_0',\dots, $\\$x_{h-1}, x_{h-1}'])[x_h, x_h']$ is also a UFD. Therefore, $p(v_1)$ and $p(v_2)$ have the same zeros, namely $-\ell^{(0)}(v_1),$\\ $-p(c_1(v_1)), \dots, -p(c_m(v_1))$ and $-\ell^{(0)}(v_2),$ $-p(c_1(v_2)) \dots, -p(c_n(v_2))$, respectively. There is a one-to-one correspondence, such that $\ell^{(0)}(v_1) = \ell^{(0)}(v_2), p(c_i(v_1)) = p(c_j(v_2))$, although the exact indices $i$ and $j$ are unknown. Thus, we can conclude that $v_1 = v_2$. By our previous assumption for nodes with height $h_1 - 1$, we have established the existence of a bijection $f: \mathcal{C}(v_1) \rightarrow \mathcal{C}(v_2)$. Furthermore, using Lemma~\ref{lemma:0-1}, we can obtain $t(v_1) \simeq t(v_2)$ and prove that it still holds for height $h_1$.
\end{proof}

\subsection{Proof of Proposition 2}

\begin{lemma}[Schwartz-Zippel-Variant~\citep{jakubowski2007software}]
Let $P\in \mathbb{Z}[x_1, \dots, x_n]$ be a (non-zero) polynomial of total degree $d > 0$ defined over the integers $\mathbb{Z}$. Let $P$ be the set of all prime numbers. Let $r_1, \dots, r_n$ be chosen at random from $\mathbb{Z}$, and $q$ is a prime number. Then ${\rm Pr}[P(r_1, \dots, r_n)$ mod $q] \leq d/q$.
\label{lemma:sz}
\end{lemma}

\noindent The proof of Proposition 2 is given below using Lemma~\ref{lemma:sz}.

\begin{proof}
Assume that $p(v_1)$ and $p(v_2)$ are two polynomials corresponding to two different complete subtrees $t(v_1)$ and $t(v_2)$, respectively. Then $p(v_1) - p(v_2) = P(X) \in \mathbb{Z}/M\mathbb{Z}[x_0, x_0', \dots, x_{h-1}, x_{h-1}']$ is a polynomial of total degree $d$, where $X \subseteq \left \lbrace x_i, x_i' \mid i \in \lbrace1, \dots, h-1\rbrace \right \rbrace$, and $d \leq |\mathcal{L}(t(v_1))|+|\mathcal{L}(t(v_2))|$. Next, we determine $X$ to compute the value of the polynomial. If we randomly choose $X$ from $\mathbb{Z}/M\mathbb{Z}$. According to the Lemma~\ref{lemma:sz}, we can obtain 
\begin{equation*}
{\rm Pr}(P({X})=0)\leq \frac{d}{M} \leq \frac{|\mathcal{L}(t(v_1))|+|\mathcal{L}(t(v_2))|}{M}.
\end{equation*}
\end{proof}

\subsection{Proof of Proposition~\ref{pro:bound}}
\begin{proof}
We evaluate the probability of at least one hash collision in generating $N$ hash values. This problem can be attributed to the famous ``Birthday Problem''~\cite{mckinney1966generalized}, in which the ideal is to distribute hash values uniformly across the given range. First, we consider this ideal case. Since each intermediate step in Eq.~(\ref{eq:hash}) takes the module of $M$, the hash values are mapped to $\mathbb{Z}/M\mathbb{Z}$. Therefore, we have a space of $M$ available hash values. When the hash function generates a new value, the space size is reduced by one. The probability of this case is given by
\begin{eqnarray*}
&& {\rm Pr}_{_{\rm Ideal}}(N) \\
=\!\!\!\!\!\!\!\!&& 1 -  1 \left(\frac{M-1}{M}\right) \left(\frac{M-2}{M}\right) \cdots \left(\frac{M-(N-1)}{M}\right) \\
=\!\!\!\!\!\!\!\!&& 1 - \left(1 - \frac{1}{M} \right) \left(1 - \frac{2}{M} \right) \cdots \left(1 - \frac{N-1}{M} \right).
\end{eqnarray*}
Since $1 - \exp(-x) \leq x$ holds when $x$ is small, we replace the corresponding factor with $1-x \leq \exp(-x)$ and obtain
\begin{eqnarray*}
&& {\rm Pr}_{_{\rm Ideal}}(N)\\
\geq \!\!\!\!\!\!\!\!&& 1 - \exp \left({-{\frac{1}{M}}}\right) \exp \left(-\frac{2}{M}\right) \cdots \exp \left(-\frac{N-1}{M}\right)\\
=\!\!\!\!\!\!\!\!&& 1 - \exp \left( -\frac{N(N-1)}{2M}\right).
\end{eqnarray*}
Next, we consider the worst case. We already know that for complete subtrees $t(v_1)$ and $t(v_2)$, the upper bound on the collision probability between the corresponding polynomial values is ${\rm Pr}_{sz} = (|\mathcal{L}(t(v_1))|+|\mathcal{L}(t(v_2))|)/M$. We choose the one with the highest ${\rm Pr}_{sz}$ and denote it as ${\rm Pr}_{sz}^{\rm max} = 2\xi/M$. The $\xi$ is the maximum number of leaves in all complete subtrees. If we have generated $i-1$ different hash values, when generating the $i$-th hash value, we want this new value to not collide with the previous $i-1$ values. Its probability is $(1-{\rm Pr}_{sz}^{\rm max})^{i-1}$. Therefore, the probability of hash collision in the worst case is
\begin{eqnarray*}
&& {\rm Pr}_{_{\rm Worst}}(N) \\
=\!\!\!\!\!\!\!\!&& 1 -  1 \left(1-{\rm Pr}_{sz}^{\rm max}\right) (1-{\rm Pr}_{sz}^{\rm max})^{2} \cdots (1-{\rm Pr}_{sz}^{\rm max})^{N-1} \\
=\!\!\!\!\!\!\!\!&& 1 - (1 - {\rm Pr}_{sz}^{\rm max})^{\frac{N(N-1)}{2}}\\
=\!\!\!\!\!\!\!\!&& 1 - \left(1 - \frac{2\xi}{M}\right)^{\frac{N(N-1)}{2}}.
\end{eqnarray*}
Furthermore, we can bound ${\rm Pr}_{_{\rm Hash}}(N)$ by
\begin{eqnarray*}
&& {\rm Pr}_{_{\rm Ideal}}(N) \leq {\rm Pr}_{_{\rm Hash}}(N) \leq {\rm Pr}_{_{\rm Worst}}(N) \\
\Longleftrightarrow  && \!\!\!\!\!\!\! 1- e^{-\frac{N(N-1)}{2M}} \! \leq {\rm Pr}_{_{\rm Hash}}(N) \leq 1 - \left(1 - \frac{2\xi}{M}\right)^{\frac{N(N-1)}{2}}\!.
\end{eqnarray*}
$M$ is limited by the range of the numerical expression of the computer. To avoid the hash collision, we take $k$ hash values for one complete subtree. In this case, a hash collision happens if all the $k$ hash values agree. This is equivalent to expanding the size of the space to the $k$-th power, that is, $M^k$. Moreover, ${\rm Pr}_{sz}^{\rm max} = (2\xi/M)^k$. By the same derivation, we obtain
\begin{eqnarray*}
1 - e^{-\frac{N(N-1)}{2M^k}} \! \leq {\rm Pr}_{_{\rm Hash}}(N) \! \leq 1 - \left(1 - \left(\frac{2\xi}{M} \right)^k \right)^{\frac{N(N-1)}{2}}\!.
\label{eq:bound2}
\end{eqnarray*}
\end{proof}

\section{Conclusions}
This paper proposed a Wasserstein graph metric with $L_1$-TED as the ground distance. Experiments showed that the WWLS could better capture slight differences in structure than the comparison methods. Since WWLS belongs to the framework of the WL test, its expressive power is equivalent to that of the WL test. An important conclusion is that although the methods have the same expressive power, adding structural information improves classification accuracy. The similarity between the WL test and the GNN mechanism suggests that the same effect is expected for GNNs. Therefore, the future challenge is to bring this idea to GNNs.

%\clearpage

\section*{Acknowledgments}

H. Kasai was partially supported by JSPS KAKENHI Grant Numbers 22K12175, and by Support Center for Advanced Telecomm. Technology Research (SCAT).

\clearpage
\bibliographystyle{unsrt}
\bibliography{WWLS}

\end{document}